\documentclass[10pt,a4paper,conference]{IEEEtran}
\usepackage{cite}

\ifCLASSINFOpdf
  \usepackage[pdftex]{graphicx}
\else
\fi
\usepackage{amsmath,amssymb,amsthm}
\usepackage{algorithm,algorithmic}
\ifCLASSOPTIONcompsoc
  \usepackage[caption=false,font=normalsize,labelfont=sf,textfont=sf]{subfig}
\else
  \usepackage[caption=false,font=footnotesize]{subfig}
\fi
\theoremstyle{theorem}
\newtheorem{prop}{Proposition}
\newtheorem*{prop*}{Proposition}
\newtheorem*{cor*}{Corollary}
\newtheorem*{remk*}{Remark}

\theoremstyle{definition}
\newtheorem*{note*}{Note}

\DeclareMathOperator{\Adam}{Adam}

\usepackage{multirow} 
\usepackage{booktabs}

\usepackage{xcolor}

\usepackage{url}

\begin{document}
%
\title{Discriminative Multi-level Reconstruction \\ under Compact Latent Space \\ for One-Class Novelty Detection}

\author{\IEEEauthorblockN{Jaewoo Park, Yoon Gyo Jung, Andrew Beng Jin Teoh}
School of Electrical and Electronic Engineering, Yonsei University, Seoul, Korea\\
Email: \{julypriase,jungyg,bjteoh\}@yonsei.ac.kr
}


\maketitle

\begin{abstract}
In one-class novelty detection, a model learns solely on the in-class data to single out out-class instances. Autoencoder (AE) variants aim to compactly model the in-class data to reconstruct it exclusively, thus differentiating the in-class from out-class by the reconstruction error. However, compact modeling in an improper way might collapse the latent representations of the in-class data and thus their reconstruction, which would lead to performance deterioration. Moreover, to properly measure the reconstruction error of high-dimensional data, a metric is required that captures high-level semantics of the data. To this end, we propose Discriminative Compact AE (DCAE) that learns both compact and collapse-free latent representations of the in-class data, thereby reconstructing them both finely and exclusively. In DCAE, (a) we force a compact latent space to bijectively represent the in-class data by reconstructing them through internal discriminative layers of generative adversarial nets. (b) Based on the deep encoder's vulnerability to open set risk, out-class instances are encoded into the same compact latent space and reconstructed poorly without sacrificing the quality of in-class data reconstruction. (c) In inference, the reconstruction error is measured by a novel  metric that computes the dissimilarity between a query and its reconstruction based on the class semantics captured by the internal discriminator. Extensive experiments on public image datasets validate the effectiveness of our proposed model on both novelty and adversarial example detection, delivering state-of-the-art performance.

%
%
%
\end{abstract}


%
\IEEEpeerreviewmaketitle

\section{Introduction}

\begin{figure*}[!t]
\centering
\subfloat[]{\includegraphics[width=0.25\textwidth]{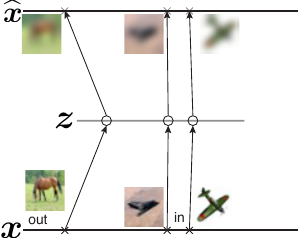}}
\hspace{5mm}
\subfloat[]{\includegraphics[width=0.25\textwidth]{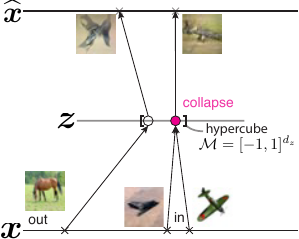}} 
\hspace{5mm}
\subfloat[]{\includegraphics[width=0.25\textwidth]{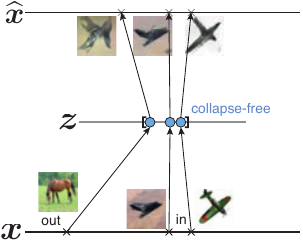}}
\caption{
(a) vanilla AE, (b) AE with its encoder output improperly constrained, (c) ours. In this schematic diagram, in the in-class data consists of plane while an horse image is given as an out-class instance. For the vanilla AE in (a), the reconstruction quality is similar for both in-class and out-class samples. In (b), the AE poorly reconstructs both the in-class and out-class instances. In both (a) and (b), effective novelty detection is not allowed as the reconstruction error does not differentiate between in-class and out-class samples. In (c), however, fine reconstruction is performed exclusively for the in-class data, allowing the model to successfully detect out-class instances by the reconstruction error.
}
\label{fig: comparison}
\vspace{-0.20in}
\end{figure*}

Novelty detection is a task to detect an incoming signal that deviates from the underlying regularity of a known class \cite{lsa}. \textit{One-class} novelty detection, in particular, assumes an additional constraint that only the known, \textit{in-class} samples are available for training \cite{ocgan}. In the inference stage, the trained system needs to detect \textit{out-class} samples, differentiating them from the in-class data. Due to the absence of out-class knowledge, the one-class novelty detection problem is of unsupervised learning and highly challenging. The range of one-class novelty detection application is diverse from medical data processing \cite{app_medical, anogan, fanogan} to intruder detection \cite{app_intruder01, app_intruder02}, abnormality detection \cite{app_abnormal}, and fraud detection \cite{app_fraud}. Moreover, novelty detection has a deep root in neuroscience \cite{neuro_science_hippo} as it is one of the core neural mechanisms of intelligent beings .

Many successful methods in novelty detection follow one of the following two approaches.
In the first strategy, a density function of the in-class data is modeled, and then a query located on the low-density region is classified as out-class \cite{low_density_rejection,density_nonparam,density_param,density_graph,density_gaussian_classifier}.
The second strategy is by reconstruction-based methods \cite{kpca,rpca,ae_tokyo,ocgan}, the core principle of which is to design a mapping that is invertible exclusively over the in-class manifold.
To differentiate in-class and out-class samples, 
the models often come together with a score function that measures the novelty of a query, which can be either
sample-wise reconstruction loss used in the training of their models \cite{ae_tokyo}, a score derived by an independent module \cite{adv_learned_classifier,ocgan}, or a mixture of them \cite{lsa}.

As to the reconstruction-based approach, most of the models follow the paradigm of compact representation learning \cite{compact_rep_kernel_two_sample} to acquire a function that reconstructs the in-class data only. Its latent representations are learned to be compact in the sense that they are so condensed as to represent the in-class data exclusively. For example, principal component analysis (PCA)-based methods \cite{bishop_pr,kpca,rpca} select a minimal number of eigen-axes by which to reconstruct the in-class data. The recent advances in deep learning \cite{deep_imagenet,deep_resnet} enabled the reconstruction-methods to learn compact representations in more diverse manners. The deep autoencoder (AE) achieves this goal by making its middle layer much lower-dimensional than its input dimension and thereby posing a bottleneck therein. Moreover, progresses in generative adversarial learning \cite{gan} enabled AE to learn compact latent representations \cite{ocgan}, showing promising results.

In OCGAN \cite{ocgan} particularly, any input is encoded to a bounded region by applying $\tanh$ activation, and all points in the region are decoded to in-class samples. Out-class instances are thus reconstructed to an in-class sample, resulting in large reconstruction error. However, the bounding constraint by $\tanh$ activation may cause collapse in the latent representations of in-class samples, causing their reconstruction in poor quality (Figure \ref{fig: comparison}(b)). As such an issue is reported in \cite{ocgan}, OCGAN seeks to resolve it by a complicated sampling technique, which, however, is insufficient to make OCGAN excel on complex dataset such as CIFAR-10, leaving room for improvement. Moreover, to infer the novelty of a given query in the inference stage, OCGAN resorts to an extra classifier on top of the dual GANs employed therein, making the model heavy.


To improve over the aforementioned issues, we propose Discriminative Compact Autoencoder, abbreviated by DCAE, that exclusively reconstructs the in-class data by learning their latent representations to be compact and collapse-free. DCAE utilizes its own internal module that captures class semantics of the in-class data  for both effective training and inference.

Our contributions are summarized as follows:

\begin{enumerate}
\item We propose to learn both \textit{compact} and \textit{collapse-free} latent representations of the in-class data so as to reconstruct them both finely and exclusively.

\item For inference, a novel measure of reconstruction error is proposed. The proposed measure evaluates the error between an input and its reconstruction by projecting onto the penultimate layer of the internal adversarial discriminator of DCAE. The projection provides the class semantics of an input query, allowing the measure to effectively differentiate the in-class from out-class. Moreover, we theoretically show that, due to Lipschitz continuity, reconstructing through multiple hidden layers of the discriminator during the training of DCAE improves the effectiveness of using the penultimate layer in inference.

\item Extensive experiments in public image data sets validate effectiveness of DCAE not only over novelty detection problem but also over the task of detecting adversarial examples, delivering the state-of-the-art results.
\end{enumerate}

We highlight that our problem to solve in this work is  \textit{unsupervised} one-class novelty detection. There are other, different settings for novelty detection: for example, semi-supervised novelty detection \cite{semi_anomaly,out_expose} allows to train with out-class data, and self-supervised novelty detection \cite{rotnet,rotnet_robust} allows a model to exploit supervisory signals inferred from a simple rule. Both settings require some amount of expert knowledge and/or human prior on a given training data. (For further discussion, see Supplementary \ref{supp_A}. In our  unsupervised one-class  setting, we only assume that a given training dataset is one-class (i.e., the known class).

\section{Method}

\begin{figure}[!tb]
\centering
\subfloat[]{\includegraphics[width=0.20\textwidth]{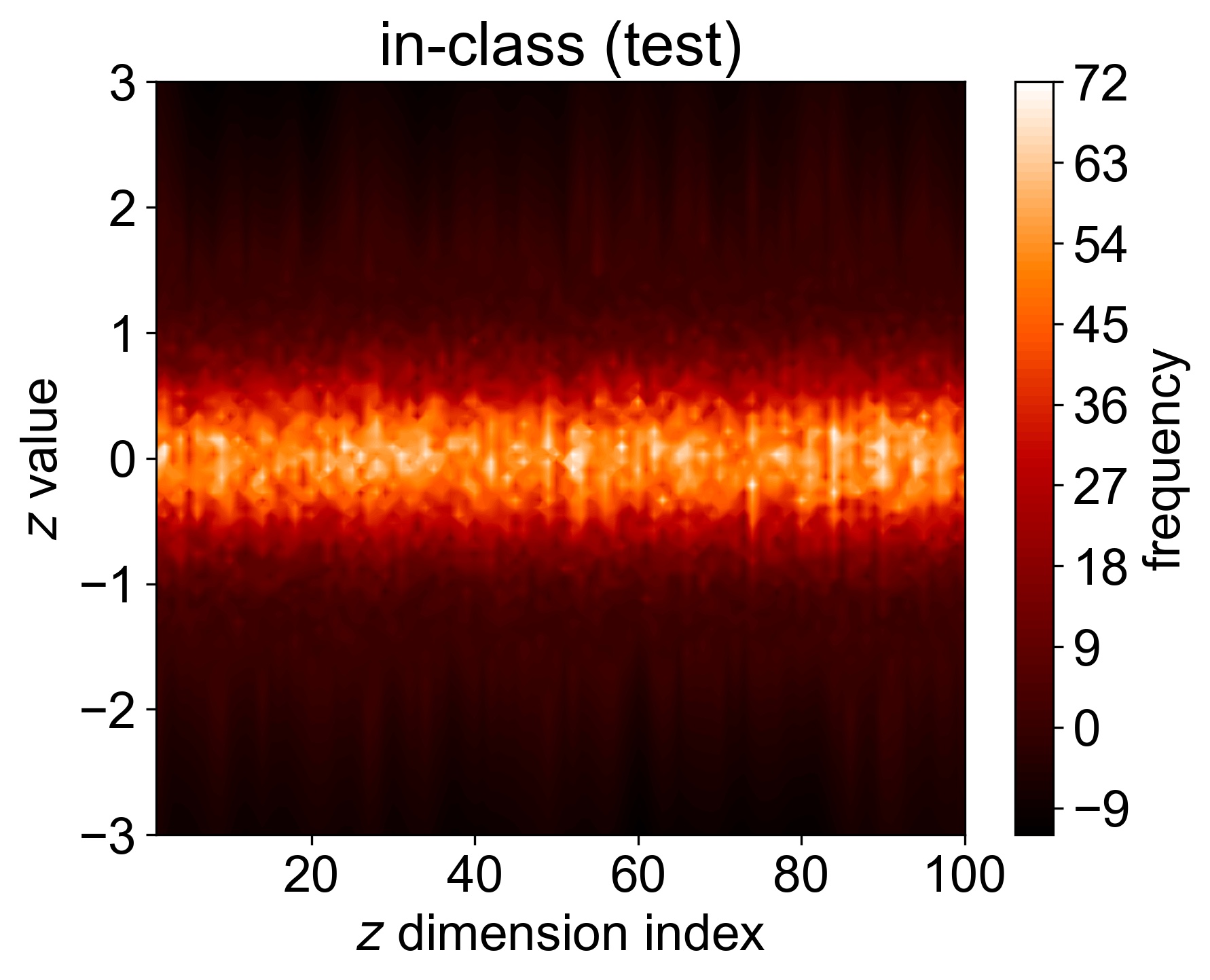}}
\hspace{0mm}
\subfloat[]{\includegraphics[width=0.20\textwidth]{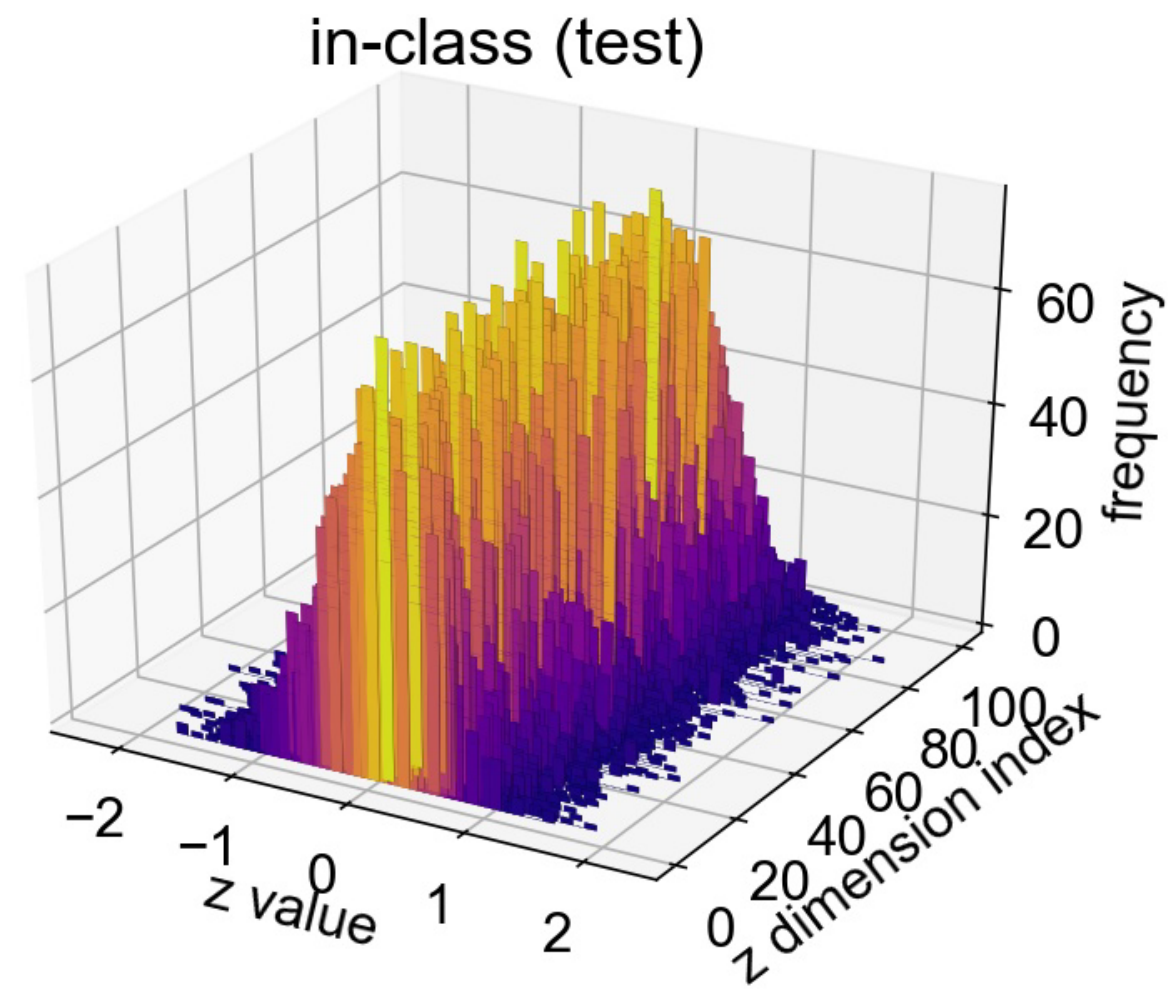}}  \\
\vspace{-0.05in}
\subfloat[]{\includegraphics[width=0.20\textwidth]{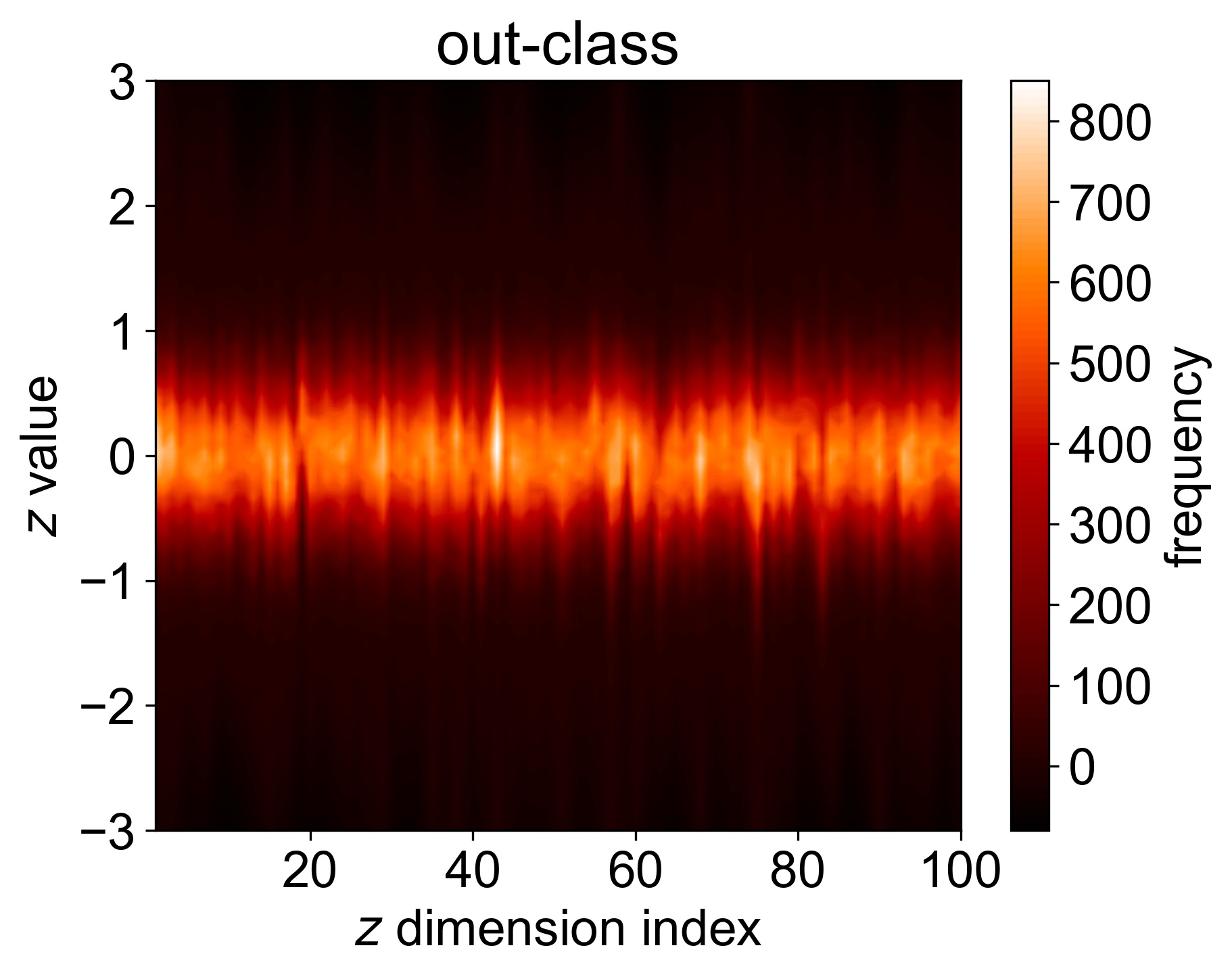}}
\hspace{0mm}
\subfloat[]{\includegraphics[width=0.20\textwidth]{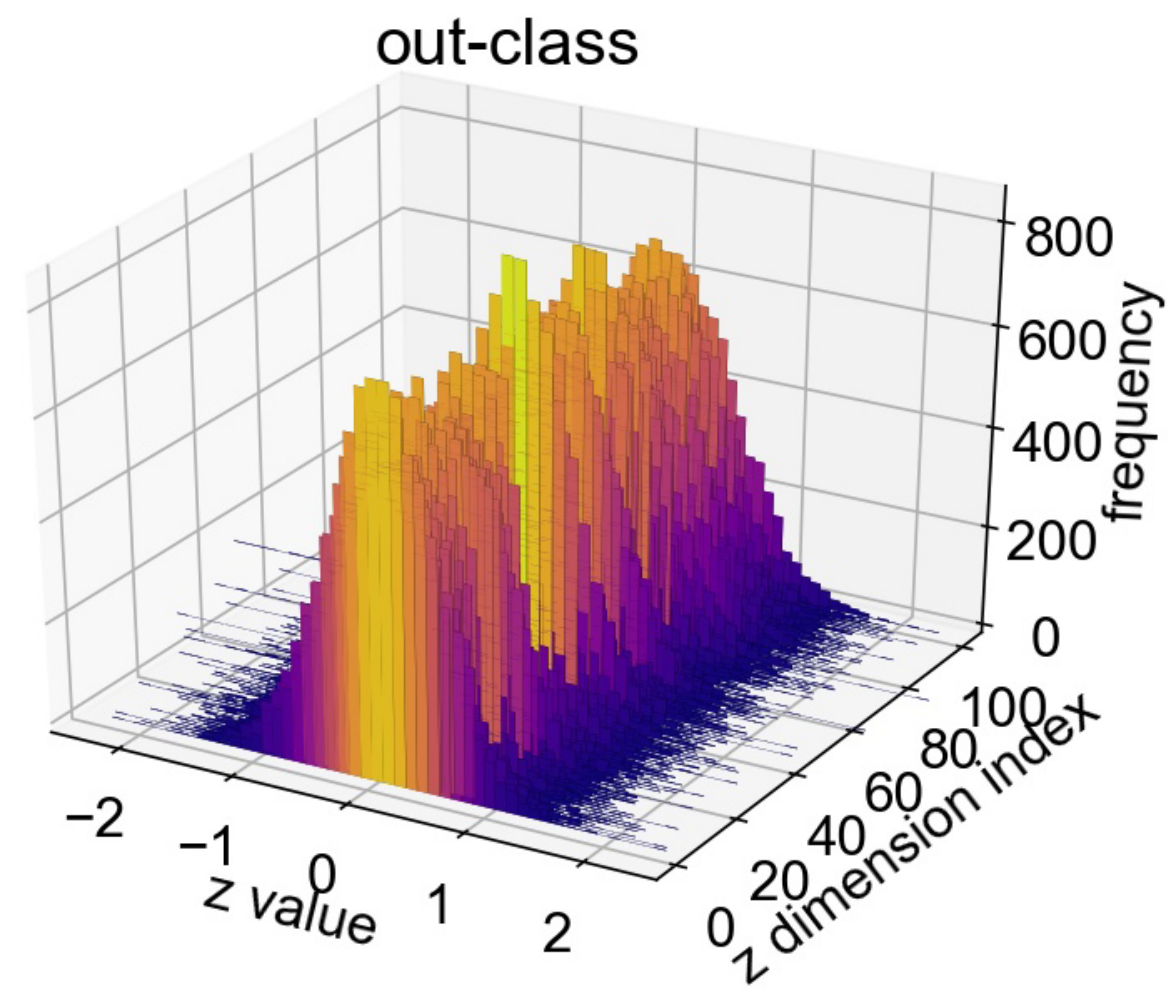}}
\caption{
Heat maps and histograms of scalar $z_k$ along the dimension index $k$ where $(z_1,\dots,z_k,\dots, z_{100}) = E(\boldsymbol{x}) \in \mathbb{R}^{100}$. For (a) and (b), $\boldsymbol{x}=\boldsymbol{x}_{in}$ are in-class test samples, and for (c) and (d) $\boldsymbol{x} = \boldsymbol{x}_{out}$ are out-class. The figures verify that the encoder outputs of not only $\boldsymbol{x}_{in}$ but also $\boldsymbol{x}_{out}$ are constrained into $\mathcal{M}=[-1,1]^{d_z}$ ($d_z =100$).
}
\label{fig: heat_map}
\vspace{-0.20in}
\end{figure}

Our method is divided into two parts: one for training the one-class model and the other for inference by measuring reconstruction error.

\subsection{Training}

The proposed model delivers to reconstruct the in-class data exclusively by (a) learning collapse-free (i.e. bijective) latent representations of the in-class data within a compact latent space and (b) constraining the latent representation of out-class instance into the same compact latent space. (a) ensures fine reconstruction of the in-class data while (b) forces the latent representation of out-class instances to be decoded to in-class samples, resulting in large reconstruction error thereof (as shown in Figure \ref{fig: comparison} (c)). For (a), firstly, the in-class data is bidirectionally represented by the compact latent space under dual GANs. Then, for fine reconstruction, they are reconstructed as projected onto multiple layers of the input discriminator. To prevent any collapse in the latent space, we linearly activates the output of the encoder and reconstruct the latent vectors. For (b), we specify our encoder to be deep so that it be vulnerable to open set risk and thus too fool to differentiate between the in-class samples and out-class instances, thereby encoding both to the same compact latent space. Due to the aforementioned bidirectional modeling of the in-class data, the latent representations of out-class instances are decoded to in-class samples, resulting in large reconstruction error of out-class instances.

The following are detailed steps to realize the desired mechanism.

\subsubsection{Bidirectional Representations of the In-Class Data by a Compact Latent Space}

To allow a given compact latent space $\mathcal{M}$, in our case a hypercube $\mathcal{M} = [-1,1]^{d_z}$, to bidirectionally represent the in-class data $\mathcal{X}_{in}$, we employ dual GANs with latent and input discriminators $D_z = D_z(\cdot; \theta_{D_z})$ and $D_x = D_x(\cdot ; \theta_{D_x})$ parametrized by $\theta_{D_z}$ and $\theta_{D_x}$, respectively. In particular, the following adversarial loss is optimized:
\begin{equation}
\underset{\theta_E,\theta_G}{\min} \,\, \underset{\theta_{D_z}, \theta_{D_x}}{\max} \,
L_{comp}=
L_{adv-z} + L_{adv-x}
\label{eq: loss_comp}
\end{equation}
where $\theta_E$ and $\theta_G$ are the weights parametrizing the encoder $E=E(\cdot;\theta_E)$ and decoder $G=G(\cdot; \theta_G)$. In Eq. \eqref{eq: loss_comp}, the latent adversarial loss $L_{adv-z}$ is defined as
\begin{multline}
L_{adv\text{-}z}(\theta_{D_z},\theta_E) \\
= -\frac{1}{N} \sum_{i=1}^{N} \left[  \log D_z(\boldsymbol{z}_i) + \log D_z(1 -E(\boldsymbol{x}_i)) \right]
\end{multline}
where $\{\boldsymbol{z}_1,\dots, \boldsymbol{z}_N\}$ and $\{\boldsymbol{x}_1,\dots,\boldsymbol{x}_N\}$ are batches sampled from the uniform prior $p_z(\boldsymbol{z})= \mathcal{U}[-1,1]^{d_z}$ and the in-class dataset $\mathcal{X}_{in}$, respectively. Under this latent adversarial loss $L_{adv-z}$,
the latent representations $E(\boldsymbol{x})$ of in-class samples $\boldsymbol{x} \in \mathcal{X}_{in}$ are mapped to and thus constrained in the compact latent space $\mathcal{M}= [-1,1]^{d_z}$. 
On the other hand, optimizing the input adversarial loss $L_{adv-x}$
\begin{multline}
L_{adv\text{-}x}(\theta_{D_x},\theta_G) \\
= -\dfrac{1}{N} \sum_{i=1}^N  
  \left[ \log D_x(\boldsymbol{x}_i) 
+ \log D_z(1 - G(\boldsymbol{z}_i)) \right]
\end{multline}
forces every latent vector $\boldsymbol{z} \in \mathcal{M}$ to represent an in-class sample through the decoder $G$; that is, $G(\boldsymbol{z}) \in \mathcal{X}_{in}$ for $\boldsymbol{z} \in \mathcal{M}=[-1,1]^{d_z}$. 

To ensure our encoder output to be collapse-free over the in-class data, we apply \textit{linear activation} rather then bounded activation such as $\tanh$.
Overall, the in-class data is bidirectionally represented by the compact latent space $\mathcal{M}$.

\subsubsection{Exclusive Representations of the In-Class Data}
To perform novelty detection effectively, the given AE must reconstruct out-class instances poorly while maintaining the reconstruction quality of the in-class samples. To this end, we specify our encoder $E$ to be a \textit{deep neural net} (DNN), and claim that the deep encoder enables our DCAE to fulfill the desired objective.

Due to the vulnerability of the deep encoder $E$ to open set risk \cite{open_set}, the encoder does not distinguish between in-class samples $\boldsymbol{x}_{in} \in \mathcal{X}_{in}$ and out-class instance $\boldsymbol{x}_{out} \in \mathcal{X}_{out}$, thereby constraining both $E(\boldsymbol{x}_{in})$ and $E(\boldsymbol{x}_{out})$ into the same compact latent space $\mathcal{M}$. To see this, note that under the latent adversarial loss $L_{adv-z}$ in Eq. \eqref{eq: loss_comp}, the encoder learns to minimize the distance between $E(\boldsymbol{x}_{in})$ and $\mathcal{M}$ 
\begin{equation}
l(\boldsymbol{x}_{in}) = d(E(\boldsymbol{x}_{in}),\mathcal{M}).
\end{equation}
Due to DNN's vulnerability to adversarial attack  and open set risk (as empirically reported in \cite{open_set,deep_fool}), the loss $l(\boldsymbol{x})$ is reduced over out-class instances $\boldsymbol{x}_{out}$ as well. In other words, the open set with small loss $l(\boldsymbol{x})<\epsilon$
\begin{equation}
\{ \boldsymbol{x} \in \mathcal{X}_{out} : l(\boldsymbol{x}) < \epsilon \}
\end{equation}
is significantly large and placed near to the in-class data $\mathcal{X}_{in}$. Therefore, the encoder output $E(\boldsymbol{x}_{out})$ of out-class instances tend to be placed inside the compact latent space $\mathcal{M}$ (as shown in Figure \ref{fig: heat_map}). As every point in the compact latent space $\mathcal{M}$ represents an in-class sample, $E(\boldsymbol{x}_{out})$ then is decoded to an in-class sample, resulting in poor reconstruction of $\boldsymbol{x}_{out}$.

\noindent \textbf{Remark.}
Both OCGAN and our DCAE poorly reconstruct out-class instance by constraining the range of latent representation $E(\boldsymbol{x})$ into a bounded space. In OCGAN, $E(\boldsymbol{x})$ is constrained explicitly by applying a bounded activation $\tanh$, possibly causing deterioration of in-class reconstruction (as depicted in \cite{ocgan}). In our case, we constrain it in a learning-based way, preventing such deterioration. The contrast is shcematically depicted in Figure \ref{fig: comparison}.

\subsubsection{Discriminative Multi-Level Reconstruction of the In-Class Data}

As our model is a reconstruction-based one, the in-class samples must be reconstructed finely. A conventional way to realize this is to minimize either $L_1$ or $L_2$ based distance $\lVert \boldsymbol{x} - \widehat{\boldsymbol{x}} \rVert$ between $\boldsymbol{x}$ and its reconstruction $\widehat{\boldsymbol{x}}:=G(E(\boldsymbol{x}))$. However, as $\boldsymbol{x}$ is high-dimensional in our case, the minimization results in blurry reconstruction, which is ineffective for our purpose.

To obtain robust reconstruction, we propose to exploit the multiple hidden layers $f_l=f_l(\cdot; \widehat{\theta}_{D_x})$ of the internal discriminator $D_x$. In particular, we minimize the distance between multi-level projections of an in-class sample and its reconstruction:
\begin{equation}
L_{eaf}(\theta_E, \theta_G) = 
\dfrac{1}{N}  \sum_{i=1}^N \left[ \sum_{l=0}^L \lVert f_l(\boldsymbol{x}_i) - f_l(\widehat{\boldsymbol{x}_i}) \rVert_1 \right].
\label{eq: eaf}
\end{equation}
Here, $f_0(\boldsymbol{x}):=\boldsymbol{x}$, and $L>0$ is the number of the hidden layers selected in $D_x$.
As a result, the reconstruction $\widehat{\boldsymbol{x}}_i$ preserves the multi-level semantics \cite{zfnet} of $\boldsymbol{x}_i$ captured by the input discriminator, necessitating fine reconstruction of the in-class data.

Moreover, \textit{multi-level reconstruction by $L_{eaf}$ during the training improves the effectiveness of using the penultimate layer $f_L$ for novelty detection in the inference stage} as will be shown in the later subsection (Proposition \ref{prop: final}).

\subsubsection{Surjective Encoding}
To fully prevent the collapse in the latent representations of the in-class data and thus deterioration of their reconstruction, the encoder $E: \mathcal{X}_{in} \to \mathcal{M}$ must be surjective. Otherwise, the encoder output of the in-class data will be constrained in a limited range of $\mathcal{M}$ and thus not fully represent the in-class data. Henceforth, we ensure the surjectivity of $E: \mathcal{X}_{in} \to \mathcal{M}$ by minimizing 
\begin{equation}
L_{inv-z}(\theta_E, \theta_G) = \dfrac{1}{N} \sum_{i=1}^N \lVert \boldsymbol{z}_i - \widehat{\boldsymbol{z}}_i \rVert
\label{eq: inv_z_loss}
\end{equation}
where $\widehat{\boldsymbol{z}} = E(G(\boldsymbol{z}))$ is the reconstruction of $\boldsymbol{z}$. Our proposition below validates that minimizing Eq. \eqref{eq: inv_z_loss} forces $E$ to be surjective.
\begin{prop}
Assume $E(\mathcal{X}_{in}) \subseteq \mathcal{M}$ and $G(\mathcal{M}) \subseteq \mathcal{X}_{in}$. If 
$
\lVert \boldsymbol{z} - \widehat{\boldsymbol{z}} \rVert_1 = 0
$
for every $\boldsymbol{z} \in \mathcal{M}$,
then $E: \mathcal{X}_{in} \to \mathcal{M}$ is surjective.
\end{prop}

On the other hand, minimzing $L_{inv-z}$ ensures the reconstruction of the inferred in-class samples $G(\boldsymbol{z})$ (i.e., the in-class samples generated by $G$) on account of the below proposition:
\begin{prop}
For any $\boldsymbol{z} \in \mathcal{M}$, 
$
\lVert G(\boldsymbol{z}) - \widehat{G(\boldsymbol{z})} \rVert_1 \leq \lVert G \rVert_{Lip} \lVert \boldsymbol{z} - \widehat{\boldsymbol{z}} \rVert_1
$
where $\widehat{G(\boldsymbol{z})} := G(E(G(z))) $ is the reconstruction of $G(\boldsymbol{z})$.
\end{prop}

Thus, if the generator $G$ is poor and the generated samples do not properly mimic real  in-class samples (i.e., $G(\boldsymbol{z}) \notin \mathcal{X}_{in}$), then $L_{inv-z}$ would harm the performance as it forces the AE to learn to reconstruct out-class samples $G(\boldsymbol{z})$. On contrary, if the generator is good, it allows the AE to learn to reconstruct on unseen in-class samples, improving novelty detection performance in test environment. 

Based on these evidences, we regulate the contribution of $L_{inv-z}$ by hyper-parametrizing $\alpha_z$.


\subsubsection{Full Objective}

\begin{figure*}[!t]
\begin{center}
   \includegraphics[width=0.6\linewidth]{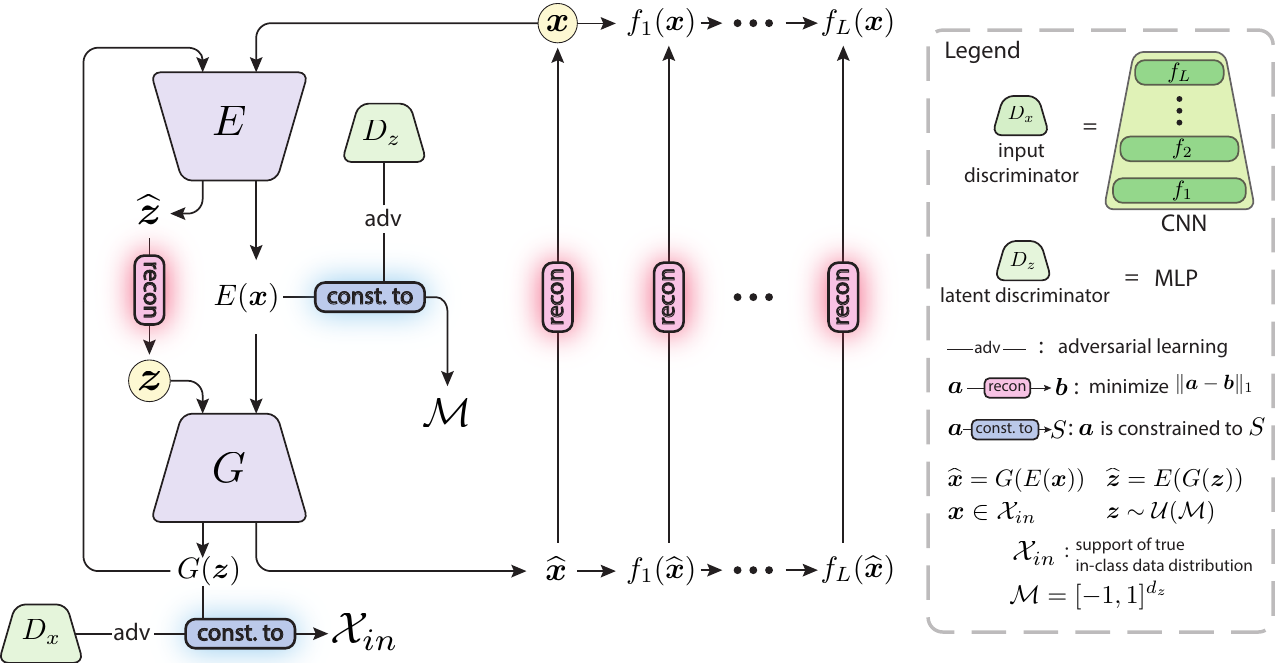}
\end{center}
\vspace{-0.20in}
\caption{A schematic diagram for the training of DCAE.}
\label{fig: loss}
\vspace{-0.2in}
\end{figure*}

The full objective of DCAE is to adversarially optimize 
\begin{equation}
\underset{\theta_E, \theta_G}{\min} \underset{\theta_{D_z}, \theta_{D_x}}{\max}\,  
L_{comp} + L_{eaf} + \alpha_z L_{inv-z} .
\label{eq: full_objective}
\end{equation}
Here, the coefficient $\alpha_z$ controls the contribution of $L_{inv-z}$. Unless specified otherwise, $\alpha_z$ is fixed to be $1$.

The overall architecture of DCAE is depicted in Figure \ref{fig: loss}, and its detailed algorithm is given in Supplementary.

\subsection{Inference}

\begin{figure}[!tb]
\centering
\subfloat[]{\includegraphics[width=0.17\textwidth]{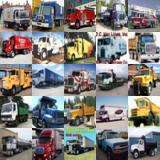}}
\hspace{4mm}
\subfloat[]{\includegraphics[width=0.17\textwidth]{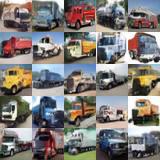}}  \\
\vspace{-2mm}
\subfloat[]{\includegraphics[width=0.17\textwidth]{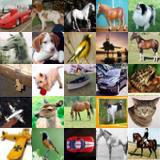}}
\hspace{4mm}
\subfloat[]{\includegraphics[width=0.17\textwidth]{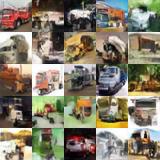}}
\caption{
Here, $\mathcal{X}_{in}$ is the truck class in CIFAR-10, and $\mathcal{X}_{out}$ is the rest of the classes. (a) test in-class $\boldsymbol{x}_{in}$, (b) test in-class reconstructions $\widehat{\boldsymbol{x}}_{in}$, (c) out-class $\boldsymbol{x}_{out}$, (d) out-class reconstructions $\widehat{\boldsymbol{x}}_{out}$.
}
\label{fig: truck}
\vspace{-0.20in}
\end{figure}

\begin{figure}[!tb]
\centering
\subfloat[]{\includegraphics[width=0.15\textwidth]{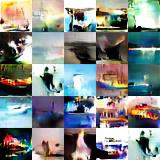}}
\hspace{1mm}
\subfloat[]{\includegraphics[width=0.15\textwidth]{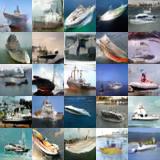}}  
\hspace{1mm}
\subfloat[]{\includegraphics[width=0.15\textwidth]{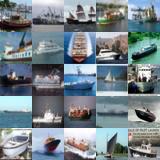}}
\vspace{-0.09in}
\caption{
Here, $\mathcal{X}_{in}$ is the ship class in CIFAR-10. (a) incorrectly generated samples $G(\boldsymbol{z}_{out})$, (b) generated samples $G(\boldsymbol{z})$, (c) real in-class $\boldsymbol{x}_{in}$. The figures show that $G(\boldsymbol{z}_{out})$ exist (i.e., can be sampled), represent out-class, but not too distant from $\mathcal{X}_{in}$.
}
\label{fig: ship}
\vspace{-0.25in}
\end{figure}

In the inference stage, a given query $\boldsymbol{x}$ is classified out-class if its nolvety score $s(\boldsymbol{x})$ exceeds a threshold $\tau$ and otherwise in-class. For reconstruction-based method, $s(\boldsymbol{x})$ is defined based on sample-wise reconstruction error.

As shown in Figure \ref{fig: truck}, the reconstruction of DCAE preserves class semantics for in-class samples but not for out-class instances. On account of this evidence, a distance metric capturing such a semantic error needs to be built to sharply classify a given query.

We claim that the penultimate layer $f_L$ of $D_x$ can be an effective building block. To see this, note that on the projected space $\{f_L(x) : x \in \mathcal{X}= \mathcal{X}_{in} \cup \mathcal{X}_{out} \}$, the in-class data $\mathcal{X}_{in}$ is linearly separated from incorrectly generated samples:
\begin{prop}
$f_L(\mathcal{X}_{in})$ is linearly separated from the projections $f_L(G(\boldsymbol{z}_{out}))$ of incorrectly generated samples $G(\boldsymbol{z}_{out})$ with $\boldsymbol{z}_{out} \in \mathbb{R}^{d_z} \setminus \mathcal{M}$ and $D_x(G(\boldsymbol{z}_{out})) < \underset{\boldsymbol{x}_{in} \in \mathcal{X}_{in}}{\min} D_x(\boldsymbol{x}_{in})$
\end{prop}
Observed in Figure \ref{fig: ship}, $G(\boldsymbol{z}_{out})$ constitutes a valid part of out-class space. The projection $f_L$ thus captures the class semantics of the in-class data accordingly.
Motivated by this, we define a novelty score $s_c$ based on the $L_1$-based error
\begin{equation}
s_c(\boldsymbol{x}) = \lVert f_L(\boldsymbol{x}) - f_L(\widehat{\boldsymbol{x}}) \rVert_1
\end{equation}
between $\boldsymbol{x}$ and $\widehat{\boldsymbol{x}}$ under the projection $f_L$.
In our ablation study, we empirically validate the effectiveness of $f_L$ compared to usage of early layers $f_l$ (Figure \ref{fig: abl_plots}(b)).

\subsubsection{Multi-level Reconstruction Improves Discriminativeness of $f_L$}
As in-class samples should have low novelty score $s_c(\boldsymbol{x})$, achieving minimal values of $s_c(\boldsymbol{x})$ on in-class samples $\boldsymbol{x}$ is crucial. Based on the following proposition, 
reconstructing through a more number of layers $f_l$ during the training of DCAE incites the model to find a smaller local minimum of $s_c(\boldsymbol{x}) =  \lVert f_L(\boldsymbol{x}) - f_L(\widehat{\boldsymbol{x}}) \rVert$ on the in-class samples $\boldsymbol{x}$.
\begin{prop}
\label{prop: final}
The reconstruction error over the final layer $f_L$ is tightly bounded by that of every previous layer:
\begin{equation}
\lVert f_L(\boldsymbol{x}) - f_L(\widehat{\boldsymbol{x}}) \rVert \leq C_l \lVert f_l(\boldsymbol{x}) - f_l(\widehat{\boldsymbol{x}}) \rVert\, \quad \forall \, l, \,x
\end{equation}
for some $C_l>0$.
\end{prop}
Overall, multi-level reconstruction by $L_{eaf}$ during training improves the inference capability of novelty score $s_c$ defined by $f_L$. The hypothesis is experimentally verified in Figure \ref{fig: abl_plots}(a).

\subsubsection{Centered Co-activation Novelty Score}
The score $s_c$ based on the $L_1$-based error might be insufficient to capture the class-relation between $\boldsymbol{x}$ and $\widehat{\boldsymbol{x}}$ since it is simply an additive ensemble of element-wise errors. In biometric models\cite{cos_sim_face}, angular distance is known to  well capture such relation. Motivated by this, we propose centered co-activation novelty score
\begin{equation}
s_{a}(\boldsymbol{x}) = 1 - a( \, f_L(\boldsymbol{x}) - m(\boldsymbol{x}) \, , \, f_L(\widehat{\boldsymbol{x}}) - \widehat{m}(\widehat{\boldsymbol{x}}) \,)
\label{eq: co_actv}
\end{equation}
where $a(\boldsymbol{x},\boldsymbol{y}) := \frac{\boldsymbol{x}^T \boldsymbol{y} }{\lVert \boldsymbol{x} \rVert_2 \lVert \boldsymbol{y} \rVert_2 } $ is the cosine similarity, and $m(\boldsymbol{x})$ is the element-wise mean of $f_L(\boldsymbol{x}) \in \mathbb{R}^{d_{f_L}}$, i.e., $m(\boldsymbol{x}) = \frac{1}{d_{f_L}} \sum_{k=1}^{d_{f_L}} f_L(\boldsymbol{x})_{k}$. 
$\widehat{m}(\widehat{\boldsymbol{x}})$ is similarly defined for $f_L(\widehat{\boldsymbol{x}})$.
The proposed score is increased if the activations in the centered feature $f_L(\boldsymbol{x}) - m(\boldsymbol{x})$ co-activate with those in the corresponding parts of $f_L(\widehat{\boldsymbol{x}}) - \widehat{m}(\widehat{\boldsymbol{x}})$. For this score to be low, not only $a(f_L(\boldsymbol{x}), f_L(\widehat{\boldsymbol{x}}))$ needs to be high but also $ \lvert m(\boldsymbol{x}) - \widehat{m}(\widehat{\boldsymbol{x}}) \rvert$ needs to be small. The latter term $\lvert m(\boldsymbol{x}) - \widehat{m}(\widehat{\boldsymbol{x}}) \rvert$ is governed by the $L_1$-based error:
\begin{prop}
For any $\boldsymbol{x}$, 
$
\lvert m(\boldsymbol{x}) - \widehat{m}(\widehat{\boldsymbol{x}}) \rvert \leq d_{f_L} s_c(\boldsymbol{x}).
$
\end{prop}
Thus, if a query $\boldsymbol{x}$ has a small $L_1$-based reconstruction error $s_c(\boldsymbol{x})$, it is reflected in the score $s_a(\boldsymbol{x})$. Overall, $s_a$ captures both the cosine and $L_1$-distance based similarities between the input and its reconstruction.

\section{Experiments}

In this section, we assess the effectivness of the proposed model DCAE. The set of experiments we conduct can be divided into three parts:
\begin{enumerate}
\item[(1)] Novelty detection performance of DCAE is evaluated on well-known benchmark data sets: MNIST \cite{dataset_mnist}, F-MNIST \cite{dataset_fmnist}, and CIFAR-10 \cite{dataset_cifar10},
\item[(2)] DCAE is applied to detect adversarial examples, tested upon GTSRB stop sign dataset \cite{dataset_gtsrb_stop},
\item[(3)] Ablation study is conducted to analyze the contribution of each component in DCAE.
\end{enumerate}

We remark that our problem is one-class unsupervised novelty detection. Thus, we do not compare with novelty detectors trained in other settings, for example, semi-supervised \cite{semi_anomaly,out_expose} and self-supervised \cite{rotnet,rotnet_robust} novelty detectors, which generally outperform unsupervised novelty detectors.

\subsection{Novelty Detection}

\begin{table}[!b]
\caption{Comparison of novelty detection performance on MNIST using Protocol B.}
\label{table: mnist}
\vspace{-0.20in}
\begin{center}
\begin{tabular}{l  l}
\specialrule{1.5pt}{1pt}{1pt}
method & AUC \\
\cmidrule{1-2}
OC-SVM \cite{ocsvm}  & 0.9513 \\
KDE \cite{bishop_pr} & 0.8143 \\
DAE \cite{low_density_rejection}  & 0.8766 \\
VAE \cite{vae} & 0.9696 \\
D-SVDD \cite{deep_svdd} & 0.9480 \\
LSA \cite{lsa} (CVPR'19) & \textbf{0.9750} \\
OCGAN \cite{ocgan} (CVPR'19) &\textbf{0.9750}\\
\cmidrule{1-2} 
\textbf{DCAE w/ $s_c$} ($\alpha_z=1.0$) & 0.9669 \\
\textbf{DCAE w/ $s_c$} ($\alpha_z=0.001$) & \textbf{0.9720} \\
\textbf{DCAE w/ $s_a$} ($\alpha_z=0.001$) & \textbf{0.9752} \\
\specialrule{1.5pt}{1pt}{1pt}
\end{tabular}
\end{center}
\vspace{-0.20in}
\end{table}

\begin{table}[t]
\caption{Comparison of novelty detection performance on F-MNIST using Protocol A.}
\label{table: fmnist}
\vspace{-0.20in}
\begin{center}
\begin{tabular}{l @{\hskip 0.3in} l}
\specialrule{1.5pt}{1pt}{1pt}
method & AUC \\
\cmidrule{1-2}
ALOCC DR \cite{adv_learned_classifier} & 0.753 \\
ALOCC D \cite{adv_learned_classifier} & 0.601 \\
DCAE \cite{ae_tokyo} & 0.908 \\
GPND \cite{gpnd} & 0.901 \\
OCGAN \cite{ocgan} & 0.924 \\
\cmidrule{1-2} 
\textbf{DCAE w/ $s_c$}  & \textbf{0.929} \\
\textbf{DCAE w/ $s_a$}  & \textbf{0.932} \\
\specialrule{1.5pt}{1pt}{1pt}
\end{tabular}
\end{center}
\vspace{-0.20in}
\end{table}

\begin{table*}[t]
\caption{
Comparision of novelty detection performance on CIFAR-10 using Protocol B.
}
\label{table: cifar10}
\vspace{-0.09in}
\begin{center}
\begin{tabular}{c c c c c c c c c c c  }
\specialrule{1.5pt}{1pt}{1pt}
& \shortstack{OC-SVM\\ \cite{ocsvm}} & \shortstack{KDE\\ \cite{bishop_pr}}& \shortstack{DAE \\ \cite{lecun_dim_reduct}} & \shortstack{VAE\\ \cite{vae}} & \shortstack{AnoGAN\\ \cite{anogan}} & \shortstack{D-SVDD\\ \cite{deep_svdd}} & \shortstack{LSA\\ \cite{lsa}} & \shortstack{OCGAN\\ \cite{ocgan}}  &  \shortstack{\textbf{DCAE} \\ w/ $s_c$} &  \shortstack{\textbf{DCAE} \\ w/ $s_a$}\\
\cmidrule{2-11}
plane &
0.630 & 0.658 &0.718 & 0.700 & 0.708  & 0.617 & 0.735 & 0.757 &  0.733 & \textbf{0.787}\\
car &
0.440 & 0.520 &0.401 & 0.386 & 0.458 & 0.659 & 0.580 & 0.531 &  0.657 & \textbf{0.737}\\
bird &
0.649 & 0.657 &0.685 & 0.679 & 0.644 & 0.508 & 0.690 & 0.640 &  0.708 & \textbf{0.734}\\
cat &
0.487 & 0.497 &0.556 & 0.535 & 0.510 & 0.591 & 0.542 & 0.620 &  \textbf{0.645} & 0.619\\
deer &
0.735 & 0.727 &0.740 & 0.748 & 0.722 & 0.609 & 0.761 & 0.723 &  \textbf{0.779}& 0.701\\
dog &
0.500 & 0.496 &0.547 & 0.523 & 0.505 & 0.657 & 0.546 & 0.620 &  0.680 & \textbf{0.722}\\
frog &
0.725 & 0.758 &0.642 & 0.687 & 0.707 & 0.677 & 0.751 & 0.723 & 0.794 & \textbf{0.800}\\
horse &
0.533 & 0.564 &0.497 & 0.493 & 0.471 & 0.673 & 0.535 & 0.575 &  0.714 & \textbf{0.773}\\
ship &
0.649 & 0.680 &0.724 & 0.696 & 0.713 & 0.759 & 0.717 & 0.820 &  0.786 & \textbf{0.823}\\
truck &
0.509 & 0.540 &0.389 & 0.386 & 0.458 & \textbf{0.731} & 0.548  & 0.554 &  0.6215 & 0.717\\
\cmidrule{2-11}
\textbf{mean} &
0.5856 & 0.6097 &0.590 & 0.5833 & 0.5916 & 0.6481 & 0.6410 & 0.6566 &  \textbf{0.7117} & \textbf{0.7412}\\
\specialrule{1.5pt}{1pt}{1pt}
\end{tabular}
\end{center}
\vspace{-0.22in}
\end{table*}

\subsubsection{Evaluation Protocol}
To assess the effectiveness of the proposed method, we test it on three well-known multi-class object recognition datasets. Following \cite{ocgan,lsa}, we conduct our experiment in a one-class setting by regarding each class at a time as the known class (in-class). The network of the model is trained using only the known class samples. In the inference stage, the other remaining classes are used as out-class samples. Based on previous works tested upon the same one-class setting, we compare our method by assessing its performance using 
Area Under the Curve (AUC) of Receiver Operating Characteristics curve. To this end, we follow two protocols widely used in the literature \cite{deep_svdd,gpnd,ocgan,lsa} of novelty detection:

\noindent \textbf{Protocol A}: Given in-class and out-class sets, 80\% of the in-class samples are used for training. The remaining 20\% is reserved for testing. The out-class samples for testing are randomly collected from the out-class set so that its total number be equal to that of the in-class test samples.

\noindent \textbf{Protocol B}: We follow the training-testing splits provided from the dataset. For training, all samples in the known class in the training set are employed. For testing, all samples in the test set are used by regarding any other class as an out-class.

%
%

\subsubsection{Results} \label{sec_results_nd}

Here, we present our results together with a brief description of the hyperparameters we used. The detailed architecture setting is given in Supplementary.

\noindent\textbf{MNIST}.
For the MNIST dataset, we tested our model upon Protocol B.  We have found that the generator generalizes too well to learn samples in the extreme (i.e., the samples near the boundary) of the in-class manifold. For this reason, we reduced the coefficient $\alpha_z$  of $L_{inv-z}$ in \eqref{eq: inv_z_loss} to $\alpha_z = 0.001$, which is known to disentangle the latent code \cite{infogan}. Our result is shown in Table \ref{table: mnist}, showing that the perforamnce is comparable to the state-of-the-art model OCGAN and LSA. Note that OCGAN needs an extra classifier module plus with dual GANs and careful sampling technique to achieve the reported performance while LSA is heavy with autoregressive density estimation and requires sensitive architectural configuration. On the other hand, our DCAE consists of dual GANs only, simply trained with a simple end-to-end loss, and inferencing by its own internal module.

\noindent\textbf{F-MNIST}.
The model performance on F-MNIST is assessed using Protocol A. Based on the MNIST experiment, we set $\alpha_z = 0.001$. 
The F-MNIST dataset is not fairly easy as there is a fair amount of intra-class variation while between some classes, the inter-class dissimilarity is not so significant (for example, 'T-shirt' and 'Pullover' classes). Our result is shown in Table \ref{table: fmnist}, showing that it outperforms the state-of-the-art OCGAN.

\noindent \textbf{CIFAR-10}
is a difficult dataset for one-class unsupervised novelty detection. Several reasons include that the dataset is fairly sparse (i.e., samples are not continuous), that it has high intra-class variation (i.e., diverse samples), and that the images are of low-resolution while they contain real objects. Table \ref{table: cifar10} shows that our model outperforms the state-of-the-art OCGAN by a large margin as tested upon Protocol B.

\subsection{Detection of Adversarial Example}

\begin{table}[t]
\caption{
The performance comparison over the task of detecting adversarial examples generated by Boundary Attack.
}
\label{table: adv_attack}
\vspace{-0.10in}
\begin{center}
\begin{tabular}{l l }
\specialrule{1.5pt}{1pt}{1pt}
method & AUC \\
\cmidrule{1-2}
OC-SVM/SVDD \cite{ocsvm} & 0.675 \\
KDE \cite{bishop_pr} & 0.605 \\
IF \cite{isolation_forest} & 0.738 \\
AnoGAN \cite{anogan} & - \\
DAE \cite{ae_tokyo} & 0.791 \\
D-SVDD \cite{deep_svdd} (ICML'18)  & 0.803 \\
\cmidrule{1-2} 
\textbf{DCAE w/ $s_a$} & \textbf{0.877} \\
\textbf{DCAE w/ $s_c$} & \textbf{0.929} \\
\specialrule{1.5pt}{1pt}{1pt}
\end{tabular}
\end{center}
\vspace{-0.20in}
\end{table}

\begin{figure}[!t]
\begin{center}
   \includegraphics[width=0.7\linewidth]{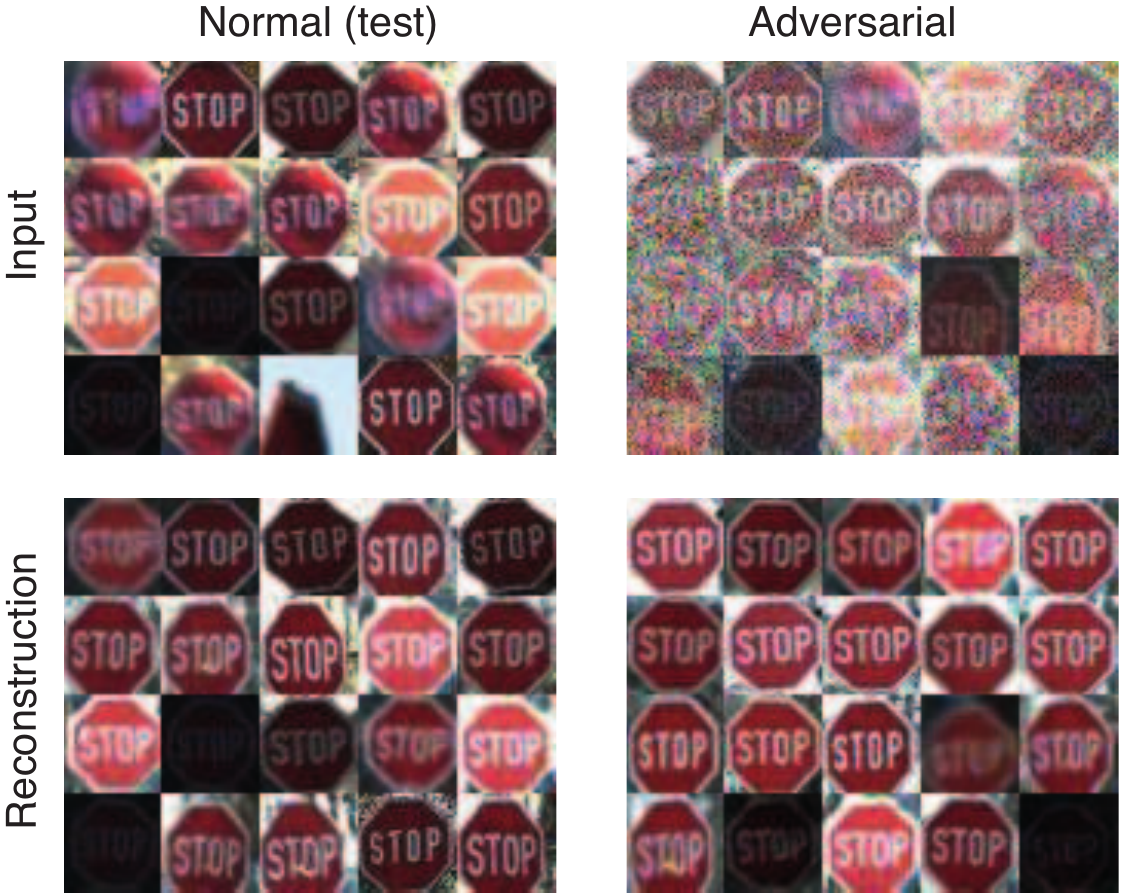}
\end{center}
\vspace{-0.20in}
\caption{Showing reconstructed images of test normal samples and adversarial examples.}
\label{fig: adv}
\vspace{-0.20in}
\end{figure}

In many practical scenarios such as security systems and autonomous driving, it is vital to detect adversarial attacks \cite{practical_attack}. In this experiment, we test our model DCAE on the task of adversarial example detection. Following the protocol proposed by \cite{deep_svdd}, we use the `stop sign' class of German Traffic Sign Recognition Benchmark (GTSRB) dataset \cite{dataset_gtsrb_stop}. The training set consists of $780$ stop sign images of spatial size $32 \times 32$. The test set is composed of $270$ stop sign images and $20$ adversarial examples, which are generated by applying Boundary Attack \cite{boundary_attack} on randomly drawn test stop sign images.

To measure the performance of our method over the task of adversarial example detection, we measured AUC over the test dataset. The model is trained solely using the training set as its in-class set. As shown in Table \ref{table: adv_attack}, our model performs effectively over this task, outperforming all baselines.

To qualitatively assess our model, we visualized the reconstructed images of the test samples. Figure \ref{fig: adv} shows that our model denoises adversarial examples as it reconstructs, resulting in poor reconstruction of them.
As to normal samples,
DCAE reconstructs them finely except noisy samples in the train data. 
It validates the effectivenss of DCAE on detecting adversarial examples.

\subsection{Ablation Study}

\begin{table}
\caption{Ablation study of the model components over CIFAR-10.}
\label{table: ablation}
\vspace{-0.10in}
\begin{center}
\begin{tabular}{l  l  l  l}
\specialrule{1.5pt}{1pt}{1pt}
model & $s_{per-pixel}$ & $s_c$ & $s_a$ 
\\
\cmidrule{1-4}
DCAE w/o $L_{eaf}$ + $L_{inv-z}$ & 0.6321 & 0.6854 & 0.7080 \\
DCAE w/o $L_{inv-z}$ & 0.6472 & 0.6926 & 0.7274 \\
\cmidrule{1-4} 
DCAE & 0.6347 & 0.7117  & 0.7412 \\
\cmidrule{1-4} 
tanh-DCAE w/o $L_{eaf} + L_{inv-z}$ & 0.6323 & 0.6617  & 0.6796 \\
\specialrule{1.5pt}{1pt}{1pt}
\end{tabular}
\end{center}
\vspace{-0.20in}
\end{table}

\begin{figure}[!tb]
\centering
\subfloat[]{\includegraphics[width=0.24\textwidth]{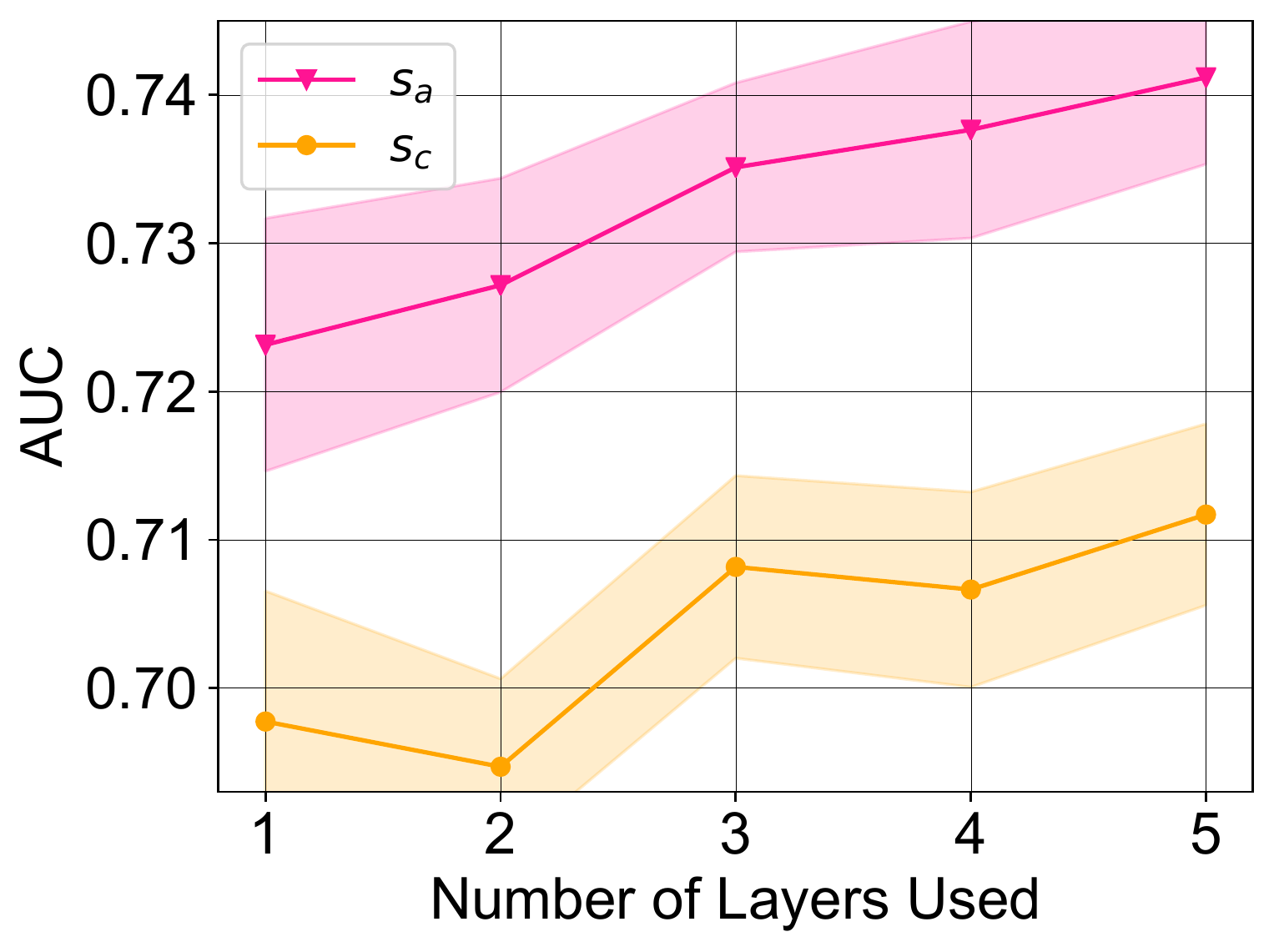}}
\hspace{0mm}
\subfloat[]{\includegraphics[width=0.24\textwidth]{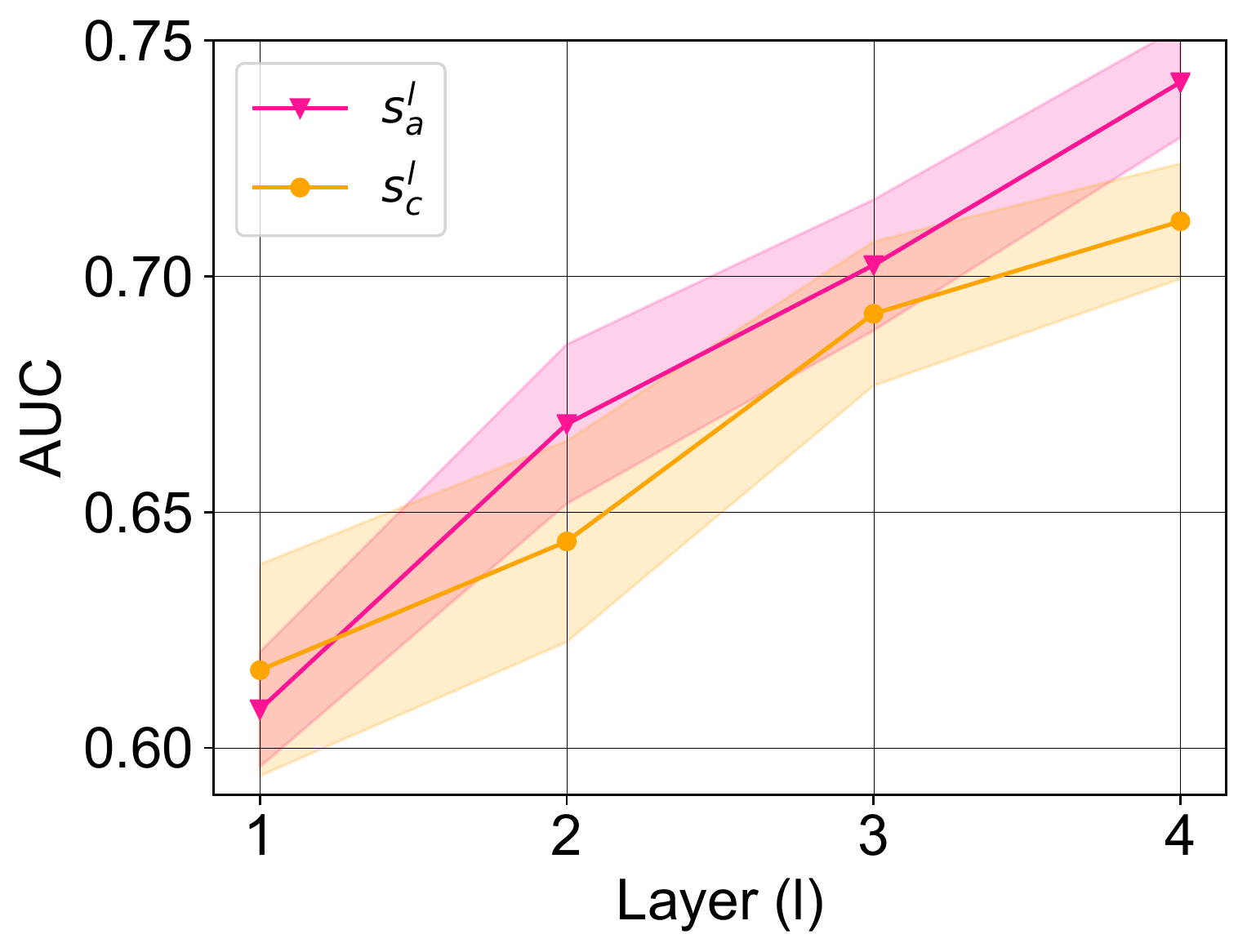}} 
\caption{
The novelty detection performances of DCAE by (a) varying $L$ in $L_{eaf}$ and (b) varying $L$ in $s_c(\boldsymbol{x})$ and $s_a(\boldsymbol{x})$, respectively. 
}
\label{fig: abl_plots}
\vspace{-0.20in}
\end{figure}

For all experiments below, we test upon CIFAR-10 based on the same protocol used above in Sec. \ref{sec_results_nd}.

\subsubsection{Ablation study on model components}

We conduct ablation study to assess the effectiveness of each component in DCAE. Our model can be decomposed into three parts that correspond to the bidirectional modeling by $L_{comp} $, multi-level reconstruction by $L_{eaf}$, and surjective encoding by $L_{inv-z}$. According to this decomposition, we consider three models: (a) DCAE without multi-level reconstruction and surjective encoding, (b) DCAE without surjective encoding, (c) full DCAE. Additionally, we test (d) DCAE without  $L_{eaf}$ and $L_{inv-z}$ but with its encoder output bounded by $\tanh$ activation. This model is equivalent to OCGAN without extra classifier and sampling technique, thus explicitly bounding the range of the latent representations and short of any of our novel approaches to resolve the latent collapse issue.

To measure the novelty detection performance of each model, we employ three different novelty scores: the score by the conventional per-pixel reconstruction error 
$
s_{per-pixel}(x):= \lVert \boldsymbol{x} - \widehat{\boldsymbol{x}} \rVert_1,
$
and the proposed novel score functions $s_c(x)$ and $s_a(x)$.

The results in Table \ref{table: ablation} show that each component of our method contributes to improving the performance of our model. An important aspect to note is that the performance improvement is not captured when the detection is performed by the per-pixel score $s_{per-pixel}$, showing both the importance and effectiveness of proper novelty score function. Moreover, the result on $\tanh$-DCAE w/o $L_{eaf} + L_{inv-z}$ shows that explicitly bounding the latent representation by applying $\tanh$ activation on the encoder output indeed degrades the performance significantly as it collapses the latent representations of the in-class data, thereby deteriorating the in-class reconstruction.

\subsubsection{On multi-Level reconstruction}
The multi-level reconstruction loss $L_{eaf}$ in \eqref{eq: eaf} has been analyzed by varying the number $L$ of ensemble components. We note that the final layer $f_L$ is always used for all cases. The result in Figure \ref{fig: abl_plots} (a) shows that the performance improves as we use larger $L$ for $L_{eaf}$ as suggested by Proposition \ref{prop: final}.

\subsubsection{Choice of hidden layer for novelty score}
We experimentally studied how the novelty detection performance changes as we define the novelty score by  another hidden layer in $D_x$. Specifically, we replaced $f_L$ by $f_l$ in both $s_c$ and $s_a$, defining 
$s^l_c(\boldsymbol{x}) = \lVert f_l(\boldsymbol{x}) - f_l(\widehat{\boldsymbol{x}}) \rVert_1$
and
$s_a^l(\boldsymbol{x}) = 1 - a(f_l(\boldsymbol{x}) - m_l(\boldsymbol{x}) , f_l(\widehat{\boldsymbol{x}}) - \widehat{m}_l(\widehat{\boldsymbol{x}}))$ with $m_l$ and $\widehat{m}_l$ defined similarly by replacing $f_L$ by $f_l$ in $m(x)$ and $\widehat{m}(x)$.
Its comparison is shown in Figure \ref{fig: abl_plots} (b), depicting a clear sign of monotonicity between the performance and the layer depth $l$. The trend validates that deeper layers of $D_x$ capture semantics more effective at differentiating the in-class data from out-class instances.

\section{Conclusion}
We proposed a recontruction-based novelty detector DCAE that induces both fine and exclusive reconstruction of the in-class data by learning its compact and collapse-free latent representations. DCAE successfully attains the desired mechanism by exploiting multi-level reconstruction based on its internal discriminator and vulnerability of the deep encoder to open set risk. Moreover, utilizing the penultimate discriminative layer and the proposed novelty score functions based on it has been validated, both theoretically and experimentally, to be effective for novelty detection inference.
Extensive experiments on public image datasets exhibited strong capability of DCAE on both novelty and adversarial example detection tasks.

\section*{Acknowledgement}
This work was supported by the National Research Foundation of Korea (NRF) grant funded by the Korea government (MSIP) (NO. NRF-2019R1A2C1003306), and by NVIDIA GPU grant program.

\bibliographystyle{IEEEtran}
\bibliography{IEEEexample}

\section{Supplementary}
Here, we provides supplementary materials including proofs for all the propositions in the main paper.
\subsection{Unsupervised vs. Self-Supervised Novelty One-Class Detectors} \label{supp_A}

Althoguh the gap between unsupervised  and self-supervised one-class novelty detectors seems not huge, there still exists a clear distinction between them. Self-supervised novelty detectors differ from unsupervised counterparts in that it requires a prior knowledge of a given dataset.
To elaborate this fact,
we show that only specific data sets can be solved by RotNet \cite{rotnet,rotnet_robust} (which achieves the top performance on object recognition image data sets for novelty detection in the self-supervised setting). In particular, we theoretically prove that RotNet completely fails as a novelty detector on in-class data sets which are closed in rotation and translation. We note that other RotNet-based variants that share the same core mechanism as RotNet thus fail on these in-class data sets in the same manner.\footnote{In fact, all other self-supervised novelty detectors are not much different from RotNet to our best knowledge.}

To prove our claim,  we denote $T_y$ to be a geometric transformation corresponding to rotation and/or translation with $y \in \mathcal{Y}=\{1,\dots, K\}$. (In RotNet of \cite{rotnet}, $T_1$ is the identity and $K=72$.) 
Then RotNet minimizes the following cross entropy loss: 

\begin{equation}
L = \sum_{i=1}^N L(\boldsymbol{x}_i) = - \sum_{i=1}^N \sum_{y=1}^K \log p(y | T_y(\boldsymbol{x}_i))
\end{equation}
where  $\boldsymbol{x}_i$ is sampled from the in-class set $\mathcal{X}_{in}$ and $p(y | T_y(\boldsymbol{x}_i))$ is the posterior given $T_y(\boldsymbol{x}_i)$. In inference, the novelty score is defined as the sum of maximal posteriors on $T_k(\boldsymbol{x})$
\begin{equation}
s(\boldsymbol{x}) = - \sum_{k=1}^K  \underset{y}{\max} \, p(y|T_k(\boldsymbol{x}))
\end{equation}
(There are other ways to define the novelty score but basically those variants are equivalent to this.)

Our claim is that RotNet completely fails if the set $\{T_y\}_{y=1}^K$ is a group that acts on $\mathcal{X}_{in}$, i.e.,
$\mathcal{X}_{in}$ is \textit{closed in}  $T_y$
\begin{equation}
\boldsymbol{x}_{in} \in \mathcal{X}_{in} \quad \implies \quad T_y(\boldsymbol{x}_{in}) \in \mathcal{X}_{in}.
\end{equation}
\begin{prop*}
If $\{T_y\}_{y=1}^K$ is a group that acts on $\mathcal{X}_{in}$, then $s(\boldsymbol{x})$ is maximal for every $\boldsymbol{x} \in \mathcal{X}_{in}$.
\end{prop*}
\begin{proof}
Since $\mathcal{X}_{in}$ is closed in $T_y$, for every $\boldsymbol{x} \in \mathcal{X}_{in}$, the posteriors $p(y=k|\boldsymbol{x})$ are minimized for all $k \in \mathcal{Y}$. This forces $p(y|\boldsymbol{x})$ to be uniform and thus induces maximal novelty score on every in-class sample $\boldsymbol{x}$.
\end{proof}

Therefore, RotNet fails on geometrically closed data sets \cite{crowdsourcing,planktonset}. For this reason, RotNet might not be suitable in applications where the data sets contain various geometrical symmetries. On the other hand, unsupervised novelty detectors (such as reconstruction-based detectors and one-class classifiers) are not constrained by such priors and thus produce consistent performance \cite{texture_anomaly,texture_one_class}.

One more aspect to note is that the application of self-supervised RotNet-based models is confined to image data sets where rotation and translation are properly defined. On the other hand, most of the reconstruction-based models including ours are applicable to general data sets.

\subsection{Algorithm of DCAE}

The detailed algorithm of DCAE is given in Algorithm \ref{alg: model}. 
Note that in the actual training, we want the contribution of the reconstruction losses $L_{eaf}$ and $L_{inv-z}$ to linearly increase. Thus, we multiply the losses by $c^{(t)}$ that linearly increases from $0$ to $1$. This is to \textit{synchronize} the reconstruction losses with the adversarial losses; adversarial learning is relatively slower than learning reconstruction.

\begin{algorithm}[!t]
\caption{Algorithm for DCAE}
\label{alg: model}
\begin{algorithmic}[1]
\REQUIRE $E, G, D_z, D_x$ with spectral normalized $\theta_E$, $\theta_G$, $\theta_{D_z}$, $\theta_{D_x}$.
In-class data $\mathcal{X}_{in}$, $T$ total iteration, $N$ batch size, $\alpha_z$, $\Adam$ minimizer

\FOR{$t=1,\dots,T$}
\STATE $\{\boldsymbol{x}_1,\dots,\boldsymbol{x}_N\} \sim \mathcal{X}_{in}$
\STATE $\{\boldsymbol{z}_1,\dots,\boldsymbol{z}_N\} \sim \mathcal{U}[-1,1]^{d_z}$
\STATE $c^{(t)} = \frac{t}{T}$
\STATE $\theta_{D_z}, \theta_{D_x}
\leftarrow \underset{\theta_{D_z}, \theta_{D_x}}{\Adam} ( - L_{comp})$
\STATE $\theta_{E}, \theta_{G}$
$\leftarrow$ $ \underset{\theta_{E}, \theta_{G}}{\Adam} ( L_{comp}  + c^{(t)} L_{aef} + c^{(t)} \alpha_z L_{inv-z})$
\ENDFOR
\RETURN The trained reconstructor $G \circ E$ and the penultimate layer $f_L$ of $D_x$
\end{algorithmic}

\end{algorithm}

\subsection{Architectures and Hyperparameters in Detail}

\begin{figure*}
\centering
\centering
\subfloat[]{\includegraphics[width=0.25\textwidth]{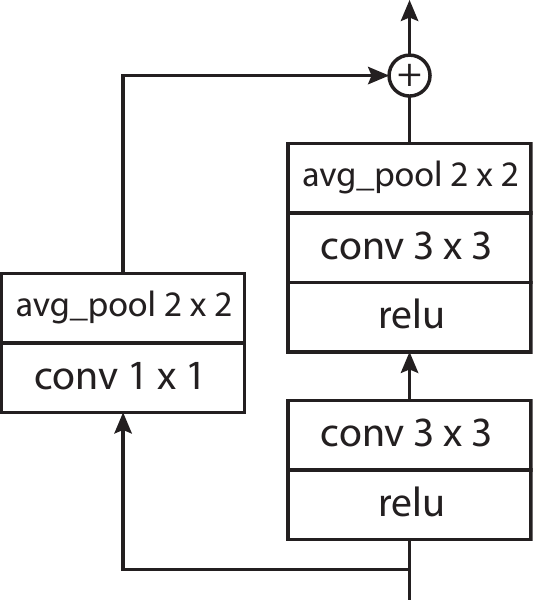}}
\hspace{5mm}
\subfloat[]{\includegraphics[width=0.25\textwidth]{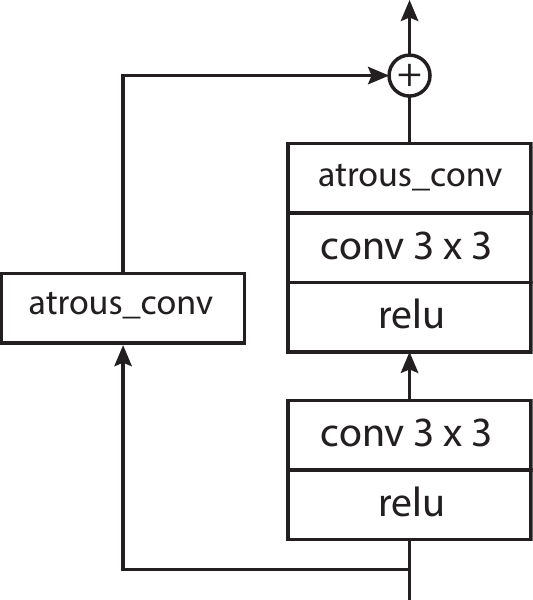}} 
\hspace{5mm}
\subfloat[]{\includegraphics[width=0.25\textwidth]{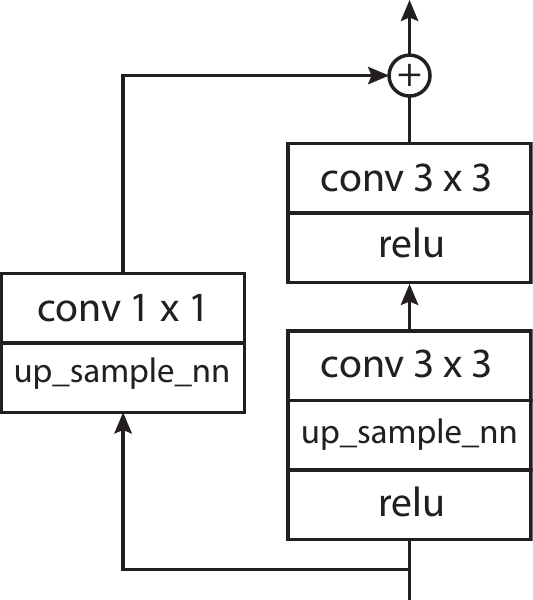}}
\caption{The residual blocks with (a) average pooling, (b) atrous convolution, and (c) nearest neighbor upsampling.}
\label{fig: res_blocks}
\end{figure*}
\begin{table}[t]
\begin{center}
\begin{tabular}{| l | l | l |}
\hline
Layer & Output & Filter  \\
\hline
FC, LeakyReLU(0.2) & $200$  & $ 100 \to 200$ \\
\hline
FC, LeakyReLU(0.2) & $200$ & $200 \to 200$ \\
\hline
FC (linear) & $1$ & $200 \to 1$ \\
\hline
\end{tabular}
\end{center}
\caption{The architecture of the input discriminator $D_z$.}
\label{table: arch_Dz}
\end{table}

\begin{table}[t]
\begin{center}
\begin{tabular}{| l | l | l |}
\hline
Layer & Output & Filter  \\
\hline
Conv & $32 \times 32 \times 32$  & $ 3 \to 32$ \\
\hline
\shortstack{ResBlock\\AvgPool} & $16 \times 16 \times 64$ & $ 32 \to 64$ \\
\hline
\shortstack{ResBlock\\AvgPool} & $8 \times 8 \times 128$ & $64 \to 128$ \\
\hline
\shortstack{ResBlock\\AvgPool} & $4 \times 4 \times 256$ & $128 \to 256$  \\
\hline
ReLU & $4 \times 4 \times 256$ &  \\
\hline
global avg pool & $1 \times 1 \times 256$ & $256 \to 256$ \\
\hline
FC (linear) & $1$ & $256 \to 1$ \\
\hline
\end{tabular}
\end{center}
\caption{The architecture of the input discriminator $D_x$.}
\label{table: arch_Dx}
\end{table}

\begin{table}[t]
\begin{center}
\begin{tabular}{| l | l | l |}
\hline
Layer & Output & Filter  \\
\hline
Conv & $32 \times 32 \times 32$  & $ 3 \to 32$ \\
\hline
\shortstack{ResBlock\\atrous Conv} & $16 \times 16 \times 64$ & $ 32 \to 64$ \\
\hline
\shortstack{ResBlock\\atrous Conv} & $8 \times 8 \times 128$ & $64 \to 128$ \\
\hline
\shortstack{ResBlock\\atrous Conv} & $4 \times 4 \times 256$ & $128 \to 256$  \\
\hline
ReLU, reshape & $4 \cdot 4 \cdot 256$ &  \\
\hline
FC (linear) & $100$ & $4\cdot 4 \cdot 256 \to 100$ \\
\hline
\end{tabular}
\end{center}
\caption{The architecture of the encoder $E$.}
\label{table: arch_E}
\end{table}

\begin{table}[t]
\begin{center}
\begin{tabular}{| l | l | l |}
\hline
Layer & Output & Filter  \\
\hline
FC (linear) & $4 \cdot 4 \cdot 256 $  & $ 100 \to 4\cdot 4 \cdot 256$ \\
\hline
reshpae & $4 \times 4 \times 256$ & \\
\hline
\shortstack{ResBlock\\UpSample} & $8 \times 8 \times 128$ & $ 256 \to 128$ \\
\hline
\shortstack{ResBlock\\UpSample} & $16 \times 16 \times 64$ & $128 \to 64$ \\
\hline
\shortstack{ResBlock\\UpSample} & $32 \times 32 \times 32$ & $64 \to 32$  \\
\hline
\shortstack{ResBlock\\UpSample} & $32 \times 32 \times 32$ & $32 \to 32$  \\
\hline
Conv & $32 \times 32 \times (3 \text{ or } 1)$ & $32 \to (3 \text{ or } 1)$ \\
\hline
\end{tabular}
\end{center}
\caption{The architecture of the encoder $G$.}
\label{table: arch_G}
\end{table}

We provide the detailed description of the architectures and hyperparameters here. All our networks are residual CNNs except the latent discriminator $D_z$, which is a multilayer perceptron with two fully connected hidden layers. To define the ensemble loss $L_{eaf}$, we pick $f_1$ from the first convolution layer and $f_2, f_3$ and  $f_4 =f_L$ from the residual block outputs.

All networks $D_x, D_z, G, E$ are spectral-normalized \cite{sngan}. To train the network we use Adam optimizer with $\beta_1= 0$ and $\beta_2 = 0$. For the learning rates, we follow TTUL \cite{ttul}, thereby setting learning rates for $(D_x,D_z)$ and $(G,E)$ differently: $\text{lr}_{D_x} = \text{lr}_{D_z} = 0.0004$ and $\text{lr}_G = \text{lr}_E = 0.0001$.
The input images are scaled to $[-1,1]$.
For all experiments, the total number of training iterations is 500K, which is relatively long but necessary to stabilize adversarial learning.

In the implementation, the hinge version \cite{sngan} of the adversarial loss is adopted.
The architectures for our networks are described in Tables \ref{table: arch_Dz}, \ref{table: arch_Dx}, \ref{table: arch_E}, and \ref{table: arch_G} with their residual blocks in Figure \ref{fig: res_blocks}. 

As to the atrous convolution in the residual block, it is a 2-dilated convolution  with no padding.
Its kernel size is chosen so as to make the spatial size of the output map to be $2$ times smaller than the input map. 

We adopt atrous convolution for for downsampling in $E$ as it makes encoding less degenerate than downsampling by average pooling. We experimentally observed that downsampling by atrous convolution in $E$ gives slightly better performance than downsampling by average pooling. However, for the discriminator $D$, average pooling was much better at stabilizing the adversarial optimization.

As to the nearest neighbor upsampling used in the residual block, its upsampling rate is fixed by $2$.

To initialize the weights of the networks, we adopt He initialization \cite{he_initialization}. We note that GANs are particularly sensitive to the choice of weight initialization. We use Adam optimizer with $\beta_1=0$ and $\beta_2 = 0.9$. The choice for these hyperparameters is motivated by recent spectral-normalized GANs \cite{cgan_proj,sagan}. 
For each iteration, we take a sample batch of size $N = 100$ for both $\boldsymbol{x}$ and $\boldsymbol{z}$.

For the penultimate layer $f_L$, we assumed that it is the final activated feature map in the derivation of our method and theory in the main paper. In the practical implementation, however, we adopted $f_L$ to be its pre-activated part $h_L$, which satisfies $ f_L= a(h_L)$. Since the activation function $a$ is ReLU, $h_L$ contains all information of $f_L$ in a linear format $h_L = [h^{-}_L, f_L]$ where $h^{-}_L$ is the negative parts of $h_L$. As the pre-activated part contains more information, it gives a slightly better performance (about 1\% higher in AUC, for example, over CIFAR-10 experiments).

\subsection{Proofs}
Here, we provide proofs of the propositions given in the main text.

\begin{proof}[Proof of Proposition 1]
Let $\boldsymbol{z} \in \mathcal{M}$. Then, $z = E(G(z))$ for any $z \in \mathcal{M}$. Since, $G(z) \in \mathcal{X}_{in}$, we obtain the desired.
\end{proof}

\begin{proof}[Proof of Proposition 2]
This is trivial because $G$ is a neural network with ReLU activations and thus Lipschitz continuous.
\end{proof}

\begin{proof}[Proof of Proposition 3]
Note that $ D_x(x) = \sigma(w^T f_L(x) + b)$ with sigmoid activation function $\sigma$, a weight vector $w \in \mathbb{R}^{d_{f_L}}$, and a bias $b \in \mathbb{R}$. By the given assumption, $w^T f_L(x_{in}) \geq w^T f_L(G(z_{out})) $ for all $x_{in}$, thus $f_L(\mathcal{X}_{in})$ is linearly separable from $f_L(G(z_{out})))$.
\end{proof}

\begin{proof}[Proof of Proposition 4]
For each layer $f_l$,
\begin{equation}
\lVert f_l(\boldsymbol{x}) - f_l(\widehat{\boldsymbol{x}}) \rVert \leq \lVert g \rVert_{\text{Lip}} \lVert f_{l-1}(\boldsymbol{x}) - f_{l-1}(\widehat{\boldsymbol{x}}) \rVert
\end{equation}
where $g$ is the function such that $f_l(\boldsymbol{x}) = g(f_{l-1}(\boldsymbol{x}))$. Here, the bound is tight. Thus, simply cascading the inequality relation through $f_{l+1}, \dots, f_L$ finishes the proof.
\end{proof}

\begin{proof}[Proof of Proposition 5]
Observe
\begin{alignat}{2}
\lvert m(\boldsymbol{x}) - \widehat{m}(\widehat{\boldsymbol{x}}) \rvert & 
= \left\vert \dfrac{1}{d_{f_L}} \left( \sum_k f_L(\boldsymbol{x})_k - \sum_k f_L(\widehat{\boldsymbol{x}})_k
\right)
\right\vert
&& \\
& 
= \dfrac{1}{d_{f_L}} \left\vert
\sum_k f_L(\boldsymbol{x})_k - f_L(\widehat{\boldsymbol{x}})_k
\right\vert
&&\\
&
\leq \dfrac{1}{d_{f_L}} \sum_k \left\vert
f_L(\boldsymbol{x})_k - f_L(\widehat{\boldsymbol{x}})_k
\right\vert
&& \\
& 
= \dfrac{1}{d_{f_L}} \lVert f_L(\boldsymbol{x}) - f_L(\widehat{\boldsymbol{x}}) \rVert_1
&& \\
& = \dfrac{1}{d_{f_L}} s_c(\boldsymbol{x})&&
\end{alignat}
by the triangle inequality of $L_1$ norm, finishing the proof.
\end{proof}








%
%
%

\end{document}